\documentclass[letterpaper, 10 pt, conference]{ieeeconf}  

\IEEEoverridecommandlockouts                              
\usepackage{cite}
\usepackage{caption}
\usepackage{amsmath,amssymb,amsfonts}
\usepackage{graphicx}
\usepackage{textcomp}
\usepackage{xcolor}
\usepackage{booktabs}
\usepackage{multirow}
\usepackage{graphicx}
\usepackage[caption=false,font=footnotesize]{subfig}
\usepackage{xspace}

\usepackage[ruled,vlined,linesnumbered]{algorithm2e}
\SetKwInput{KwIn}{Input}
\SetKwInput{KwOut}{Output}
\SetKwComment{Comment}{$//$ }{}
\DontPrintSemicolon

\newtheorem{theorem}{Theorem}
\newtheorem{assumption}{Assumption}
\newtheorem{lemma}{Lemma}
\newtheorem{definition}{Definition}

\newcommand{\ie}{\textit{i.e.}\xspace}

\DeclareMathOperator{\dimn}{dim}

\title{\LARGE \bf
Asymptotically Optimal Multi-Robot Task and Motion Planning
}

\author{
Thi Thuy Ngan Duong$^{1}$,
Cheuk Tung Shadow Yiu$^{1}$,
Rahul Shome$^{2}$,
and Yoonchang Sung$^{1\dagger}$%
\thanks{$^{1}$Nanyang Technological University, $^{2}$The Australian National University}
\thanks{$^{\dagger}$Corresponding author: {\tt\small yoonchang.sung@ntu.edu.sg}}%
}

\begin{document}

\maketitle
\thispagestyle{empty}
\pagestyle{empty}

\begin{abstract}
Multi-robot task and motion planning (MR-TAMP) requires jointly reasoning about discrete task decisions and continuous collision-free motions of multiple interacting robots. Although asymptotically optimal algorithms have been developed for task and motion planning, extending these guarantees to the multi-robot setting introduces an important challenge: different task transitions may involve different subsets of robots and therefore impose constraints of different dimensions on the composite configuration space. Consequently, an asymptotically optimal planner must not only optimize motion within each task mode, but also ensure sufficient exploration of the different types of transitions connecting them. We characterize this transition structure and establish sufficient conditions for global asymptotic optimality in MR-TAMP, requiring persistent coverage of relevant transitions and asymptotically improving motion planning within connected feasible regions. Based on these conditions, we develop an efficient asymptotically optimal MR-TAMP algorithm that combines evolving individual-robot roadmaps with implicit tensor-product search, avoiding explicit construction of the composite roadmap. The planner further employs conditional transition sampling, lazy collision checking, and mode- and solution-level guidance to improve finite-time planning efficiency while retaining persistent exploration. The resulting framework provides asymptotic optimality guarantees for multi-robot manipulation while efficiently exploiting the structure of individual-robot motion planning.

\end{abstract}

\section{Introduction}

Multi-robot manipulation requires making decisions at two tightly coupled levels. At the task level, a planner must determine which robots perform which operations and in what order or concurrency pattern. At the motion level, it must generate collision-free trajectories that satisfy manipulation constraints and coordinate all robots in a shared workspace. The \emph{multi-robot task and motion planning} (MR-TAMP) problem is naturally multi-modal: discrete changes such as grasping, placing, handing off, or releasing objects alter the continuous constraints governing subsequent motion (Fig.~\ref{fig:first_figure}).

A central question is whether an MR-TAMP planner can provide an \emph{asymptotic optimality} guarantee: as planning effort increases, the cost of the best returned solution converges to the globally optimal MR-TAMP cost. Prior work has established asymptotic optimality for integrated task and motion planning by analyzing planning over multiple continuous manifolds connected through transition sets~\cite{shome2020pushing}. Extending this analysis to MR-TAMP, however, requires exposing additional geometric structure. In particular, a transition may involve one robot, several robots, or all robots simultaneously. Consequently, different portions of a mode boundary can be defined by different numbers of independent constraints and can therefore have different intrinsic dimensions.

For example, consider two robots that independently need to grasp two objects. A transition may correspond to only the first robot satisfying its grasp constraint, only the second robot satisfying its grasp constraint, or both robots satisfying their constraints simultaneously. The simultaneous transition generally lies on a lower-dimensional set than either single-robot transition. Thus, the transition boundary is not naturally represented as a single smooth manifold of fixed dimension. Instead, it consists of manifold pieces of potentially different dimensions. We refer to each such piece as a \emph{transition stratum}, and their union as a \emph{stratified transition set}.

\begin{figure}
    \centering
    \includegraphics[width=1.00\linewidth]{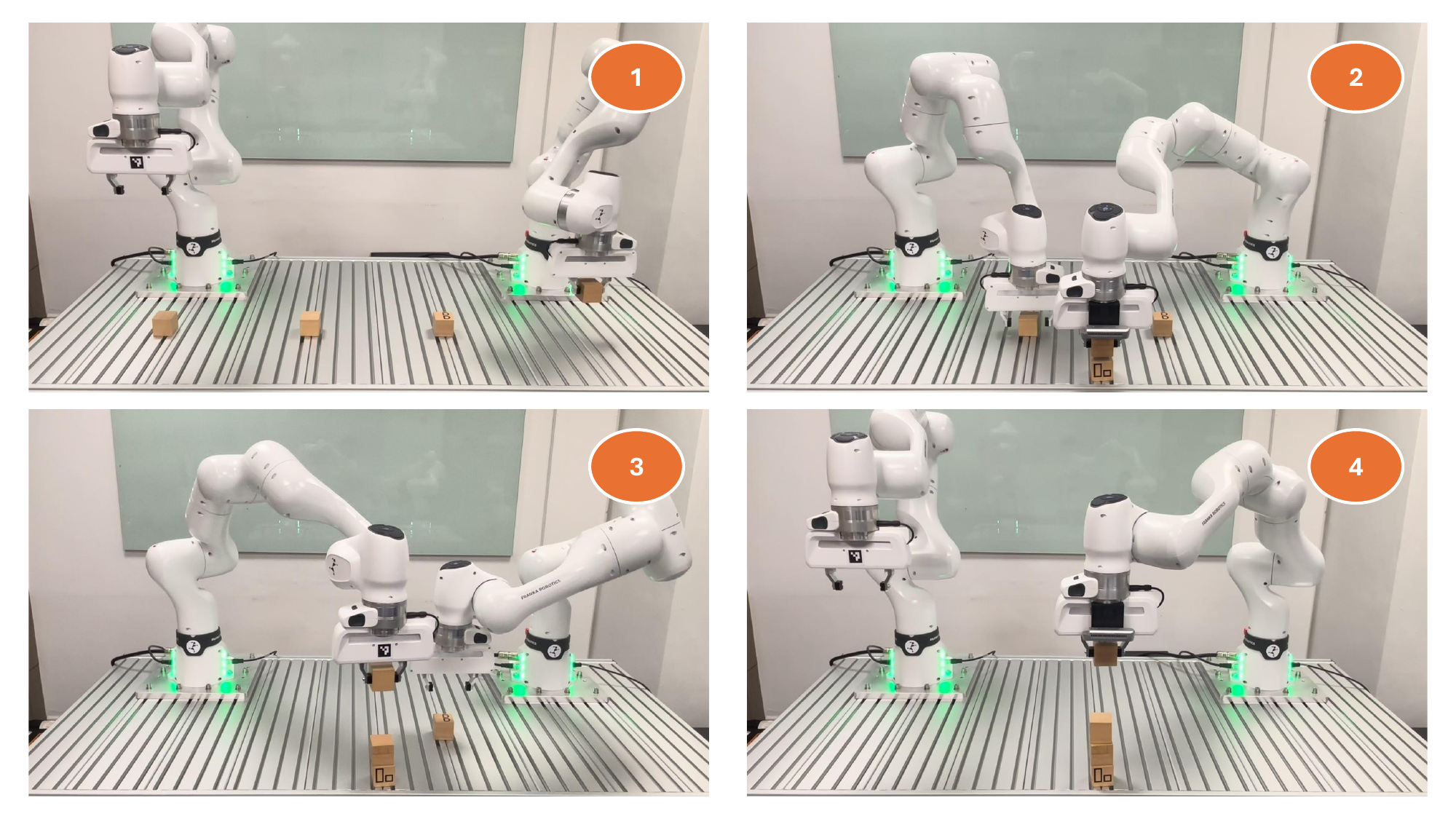}
    \caption{\textbf{Multi-robot task and motion planning for cooperative manipulation.} Multiple robots must jointly reason about task assignments and manipulation order while planning collision-free motions through changing modes. }
    \label{fig:first_figure}
\end{figure}

This observation has a direct consequence for asymptotically optimal planning. Each transition stratum is lower-dimensional than the composite configuration space and therefore has zero probability of being sampled exactly by ordinary full-dimensional sampling. Moreover, because different subsets of robots induce strata of different dimensions, a transition sampler designed for only one type of transition may fail to cover solutions requiring another. An AO MR-TAMP planner must therefore explicitly cover every transition stratum that can participate in an optimal solution. Building on prior multi-modal AO-TAMP analysis~\cite{shome2020pushing}, we characterize this geometric structure and show that conditions on orbital motion approximation, nonzero-probability coverage of the required transition strata, and persistent exploration of the corresponding orbit sequence are sufficient for global asymptotic optimality.

These theoretical conditions motivate our planning algorithm. We represent the discrete-continuous structure using an \emph{orbital graph}, whose vertices are \emph{orbits}, \ie, path-connected components of the feasible configuration space associated with a fixed mode, and whose edges correspond to feasible modal transitions. A conditional transition sampler selects the subset of robots participating in a transition and samples directly from the corresponding transition stratum, ensuring the coverage required by the analysis.

Efficient continuous planning remains challenging because explicitly constructing the composite multi-robot roadmap is prohibitively expensive. Instead, we incrementally construct constituent roadmaps for the individual robots while ignoring robot--robot interactions, and search their tensor product only implicitly. Inspired by the product-roadmap search principle of dRRT$^*$~\cite{shome2020drrt}, we maintain a \emph{tensor-product search tree} that explores only relevant combinations of individual-robot roadmap states without explicitly enumerating the composite roadmap. The constituent roadmaps continue to evolve as new modes and orbits are discovered. We further employ lazy collision checking, deferring single-robot motion-validity checks until the corresponding motions are considered during tensor-product tree expansion and checking robot--robot interactions only for explored composite motions.

We additionally exploit the multi-modal structure to guide computational effort. Before a complete solution is found, \emph{optimistic modal guidance} prioritizes exploration toward modes estimated to be closer to a goal mode. Once a solution is available, \emph{orbit-path guidance} allocates more effort to the goal-reaching orbit sequence to progressively improve its cost. Both forms of guidance are mixed with a persistent exploration policy, thereby improving finite-time efficiency while ensuring that relevant alternative orbits continue to receive planning effort.

The contributions of this work are:
\begin{itemize}
    \item We characterize the geometric structure of MR-TAMP mode transitions and identify sufficient conditions for global asymptotic optimality, including explicit coverage of transition strata induced by different subsets of transitioning robots.

    \item We develop an efficient AO MR-TAMP algorithm that realizes these conditions using evolving individual-robot roadmaps, implicit tensor-product search, lazy collision checking, and mode- and orbit-level guidance to improve finite-time planning efficiency.
\end{itemize}

\section{Related Work}

Our work lies at the intersection of TAMP, MR-TAMP, and multi-robot motion planning, with particular emphasis on multi-modal structure and asymptotic optimality.

\subsection{Task and Motion Planning}

TAMP jointly reasons over discrete task decisions and continuous robot motions. Representative approaches include hierarchical~\cite{kaelbling2011hierarchical}, optimization-based~\cite{toussaint2015logic}, constraint-based~\cite{dantam2018incremental}, and sampling-based methods~\cite{garrett2018ffrob,garrett2020pddlstream}; a broader review is available in the literature~\cite{garrett2021integrated}. Multi-modal formulations are particularly relevant to manipulation, where grasping, placing, and releasing objects change the continuous constraints governing feasible motion~\cite{hauser2010multi}.

Several works study asymptotic optimality in TAMP and related multi-modal planning problems. Vega-Brown and Roy~\cite{vega2020asymptotically} consider planning under piecewise-analytic constraints, while Shome \emph{et al.}~\cite{shome2020pushing} establish asymptotic optimality for integrated multi-modal TAMP and provide the main theoretical foundation for our analysis. TMIT$^*$~\cite{thomason2022task} further develops an almost-surely asymptotically optimal integrated TAMP planner. Our work extends this line of analysis to the transition structure induced by multiple robots.

\subsection{Multi-Robot Task and Motion Planning}

MR-TAMP additionally requires reasoning about robot assignment, concurrency, and inter-robot geometric constraints. Prior methods address task decomposition and subtask dependencies~\cite{motes2020multi}, scalable hypergraph-based decomposition~\cite{motes2023hypergraph,lee2026lazy}, long-horizon rearrangement and assembly~\cite{hartmann2022long}, collaborative manipulation and handovers~\cite{zhang2023multi}, asynchronous task-plan refinement~\cite{sung2024asynchronous}, and integrated temporal scheduling and motion planning~\cite{garrett2026schedulestream}. These methods primarily target efficient computation of feasible or high-quality coordinated plans.

Hartmann \emph{et al.}~\cite{hartmann2025sampling} consider multi-modal, multi-robot, multi-goal planning in the composite configuration space and use asymptotically optimal sampling-based planners for the underlying continuous motion problems. In contrast, we study sufficient geometric and exploration conditions for \emph{global} asymptotic optimality in MR-TAMP. In particular, transitions involving different subsets of robots can occupy strata of different intrinsic dimensions, requiring explicit coverage of the relevant transition strata.

\subsection{Multi-Robot Motion Planning}

Multi-robot motion planning has been extensively studied in both discrete and continuous settings. Multi-agent path finding (MAPF) methods such as M* and conflict-based search reduce joint-state search complexity by exploiting sparse interactions among robots~\cite{wagner2015subdimensional,sharon2015conflict}. Continuous approaches similarly exploit individual-robot structure through selective coupling and implicit product-roadmap search~\cite{otte2018dynamic,solovey2016finding,shome2020drrt}. In particular, dRRT$^*$~\cite{shome2020drrt} asymptotically optimizes paths over an implicitly represented tensor product of individual-robot roadmaps. Our planner builds on this structure for motion planning within each orbit, while additionally reasoning over the multi-modal transitions connecting different orbits.

\section{Problem Formulation}
\label{sec:problem}

Let $\mathcal{R}=\{r_1,\ldots,r_R\}$ denote a set of $R$ robots. Robot $r$ has configuration space $\mathcal{C}_r$, and the composite robot configuration is $\mathbf{q}=(q_1,\ldots,q_R)\in \mathcal{C} = \prod_{r\in\mathcal{R}}\mathcal{C}_r.$ The system may additionally contain movable objects. For conciseness, their continuous states are included in the mode-specific configuration representation when necessary.
\begin{definition}[Mode]
A mode $\mathcal{M}$ specifies the discrete manipulation state of the system, represented by the robot--object attachment relations.
Each mode induces a corresponding feasible configuration space $\mathcal{C}^{\mathcal{M}}_{\mathrm{free}}\subseteq\mathcal{C}$,
consisting of configurations that satisfy the attachment constraints and are collision-free.
\end{definition}

For intuition, classical manipulation planning distinguishes \emph{transit} motion, where the robot moves while an object remains at a stable placement, from \emph{transfer} motion, where the robot moves while holding the object at a fixed grasp~\cite{simeon2004manipulation}. These correspond to different modes because they impose different kinematic constraints; MR-TAMP generalizes this notion to assignments involving multiple robots and objects.

A mode need not be path connected. We therefore distinguish its connected feasible components.

\begin{figure}
    \centering
    \includegraphics[width=0.99\linewidth]{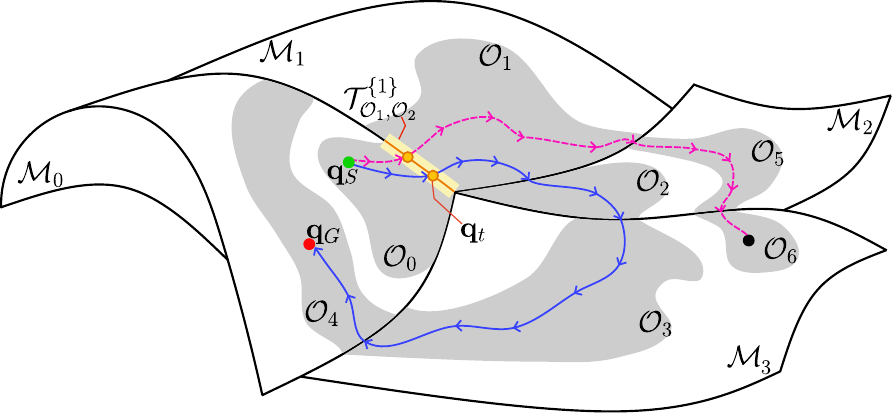}
    \caption{\textbf{Modes, orbits, and transition manifolds.} A simple two-robot example in which each robot transports an object to its respective goal. In mode $\mathcal{M}_0$, both robots are in transit without carrying objects; in $\mathcal{M}_1$, Robot~1 carries Object~1; in $\mathcal{M}_2$, both robots carry their respective objects; and in $\mathcal{M}_3$, only Robot~2 carries Object~2. The valid path (blue) transitions between modes and their corresponding orbits to reach the single goal configuration $\mathbf{q}_G$. The invalid path (purple) reaches a dead-end configuration, where the selected object placement renders subsequent transitions infeasible.}
    \label{fig:manifolds}
\end{figure}

\begin{definition}[Orbit]
An orbit $\mathcal{O}$ is a path-connected component of $\mathcal{C}^{\mathcal{M}}_{\mathrm{free}}$ for some mode $\mathcal{M}$. Any two configurations in the same orbit can be connected by a continuous feasible path that remains within the same mode.
\end{definition}

Fig.~\ref{fig:manifolds} illustrates the resulting mode, orbit structure for a two-robot manipulation example, where different object
attachments induce different modes, each of which may contain multiple orbits.

Thus, a single mode may contain multiple orbits that cannot be connected without leaving the mode. A feasible modal transition occurs at a configuration shared by the closures of two incident orbits.

Let $\mathbf{q}_s$ be the initial configuration and $\mathcal{Q}_g$ the goal set. An MR-TAMP solution is a continuous trajectory $\Pi:[0,1]\rightarrow \bigcup_{\mathcal{M}} \mathcal{C}^{\mathcal{M}}_{\mathrm{free}}$ that starts at $\mathbf{q}_s$, terminates in $\mathcal{Q}_g$, and changes modes only through feasible transition configurations.

A solution can be decomposed as
\begin{equation}
\Pi =
\pi_0 \circ \mathbf{t}_1 \circ \pi_1
\circ \cdots \circ
\mathbf{t}_K \circ \pi_K,
\label{eq:path-decomp}
\end{equation}
where each $\pi_k$ is a feasible path contained in an orbit
$\mathcal{O}_k$, and each $\mathbf{t}_k$ is a transition configuration
connecting two consecutive orbits.

Let $J(\Pi)$ denote the solution cost. We assume that the objective is compatible with a decomposition into orbital-motion and transition costs:
\begin{equation}
J(\Pi) = \sum_{k=0}^{K} J_{\mathcal{O}_k}(\pi_k) + \sum_{k=1}^{K} J_T(\mathbf{t}_k).
\label{eq:additive-cost}
\end{equation}
The objective is to find $\Pi^*=\arg\min_{\Pi} J(\Pi),$ with optimal cost $J^*=J(\Pi^*)$.

\section{Geometry of Multi-Robot Mode Transitions}
\label{sec:geometry}

The distinguishing geometric feature of MR-TAMP is that a modal transition may constrain an arbitrary nonempty subset of robots. Let $A\subseteq\mathcal{R}$, 
where $A\neq\emptyset$, be the \emph{active transition set}: the robots whose configurations satisfy new transition constraints at a particular modal change.

Consider two incident orbits $\mathcal{O}$ and $\mathcal{O}'$. For an active set $A$, define
\begin{equation}
\mathcal{T}^{A}_{\mathcal{O},\mathcal{O}'} = \left\{\mathbf{q}\in \overline{\mathcal{O}} \cap \overline{\mathcal{O}'} : h_A(\mathbf{q})=0 \right\},
\label{eq:transition-stratum}
\end{equation}
where $\overline{\mathcal{O}}$ denotes the closure of orbit $\mathcal{O}$, including its boundary configurations, and $h_A$ collects the transition constraints associated with the robots in $A$. The use of orbit closures is important because a transition configuration generally lies on the boundary between two incident orbits rather than in the interior of either one. We refer to each connected smooth component of $\mathcal{T}^{A}_{\mathcal{O},\mathcal{O}'}$ as a \emph{transition stratum}.

Suppose the transition constraints associated with each robot $r\in A$ have rank $k_r$, and that these constraints are jointly independent, so that the combined transition map $h_A$ satisfies $\operatorname{rank} Dh_A = \sum_{r\in A} k_r.$ Then,
\begin{equation}
\dimn \left( \mathcal{T}^{A}_{\mathcal{O},\mathcal{O}'} \right) = \dimn(\mathcal{O}) - \sum_{r\in A} k_r.
\end{equation}
Thus, transition sets associated with different active robot subsets can have different intrinsic dimensions.


The complete transition boundary is therefore naturally represented as
\begin{equation}
\mathcal{T}_{\mathcal{O}} = \bigcup_{\mathcal{O}'} \; \bigcup_{\emptyset\neq A\subseteq\mathcal{R}} \mathcal{T}^{A}_{\mathcal{O},\mathcal{O}'}.
\label{eq:stratified-union}
\end{equation}

The relevant consequence is not merely terminological. Each transition stratum has lower dimension than its ambient orbit and hence zero measure under an ordinary sampler defined over the full-dimensional configuration space. Furthermore, strata indexed by different active subsets need not have the same dimension. An MR-TAMP planner that is required to discover all relevant feasible transition sequences must therefore sample these sets explicitly.

\subsection{Conditional Stratified Transition Sampling}

We use a mixture sampler over transition strata. For a selected orbit $\mathcal{O}$, let $p_{\mathcal{O}}$ denote a distribution over active robot subsets. The sampler first chooses
\begin{equation}
A\sim p_{\mathcal{O}}(A),
\label{eq:subset-prob}
\end{equation}
and then samples a transition configuration from a conditional distribution
$\mu_{\mathcal{O},A}$:
\begin{equation}
\mathbf{t} \sim \mu_{\mathcal{O},A}, \qquad \mathbf{t}\in \mathcal{T}^{A}_{\mathcal{O}},
\label{eq:conditional}
\end{equation}
where $\mathcal{T}^{A}_{\mathcal{O}} = \bigcup_{\mathcal{O}'}
\mathcal{T}^{A}_{\mathcal{O},\mathcal{O}'}$.
The distribution $\mu_{\mathcal{O},A}$ may be realized using task-specific constrained sampling, inverse kinematics, projection, or another sampler defined with respect to the intrinsic measure of the corresponding stratum.

For asymptotic coverage, equal probability across active subsets is
unnecessary. It is sufficient that every stratum that may be required by an optimal solution receives persistent positive sampling probability. This allows the planner to bias sampling toward more promising transitions while retaining the coverage required for asymptotic optimality.

\section{Asymptotically Optimal MR-TAMP}
\label{sec:algorithm}

We now present the proposed MR-TAMP planner. The planner maintains an \emph{orbital graph} $\mathcal{G}_O$, whose vertices are discovered orbits and whose edges represent feasible modal transitions, together with a \emph{modal graph} $\mathcal{G}_M$ used to guide exploration toward the goal. Fig.~\ref{fig:orbital_tree} illustrates an example orbit path through the
orbital graph from the initial orbit to the goal orbit. Transition configurations are generated using the conditional stratified sampler of Section~\ref{sec:geometry}. In Algorithm~\ref{alg:ao-mrtamp}, $\mathrm{SampleStratifiedTransitions}{}$ implements this sampler by first sampling an active robot set $A\sim p_{\mathcal{O}}(A)$ and then sampling from the corresponding transition stratum.

\begin{figure}
    \centering
    \includegraphics[width=0.99\linewidth]{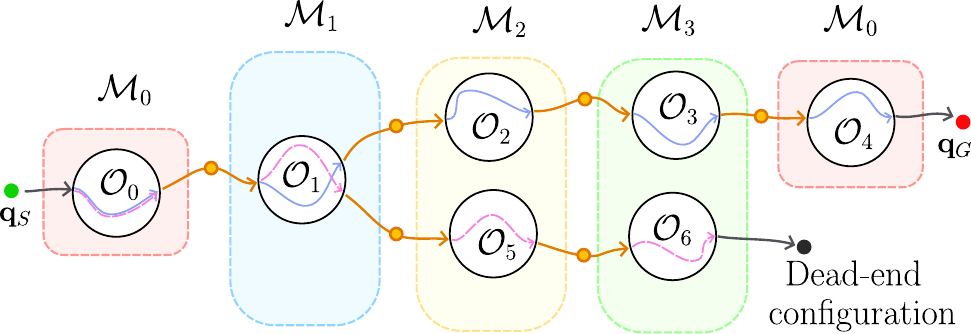}
    \caption{
    \textbf{Modal and orbital graphs} corresponding to the multi-modal configuration space in Fig.~\ref{fig:manifolds}. The modal graph captures mode transitions, while the orbital graph distinguishes path-connected feasible components within each mode.
    }
    \label{fig:orbital_tree}
\end{figure}

Explicitly constructing a roadmap in the composite configuration space is prohibitively expensive. Instead, each robot $i$ maintains an evolving \emph{constituent roadmap} $G_i$, constructed independently of the other robots. Unlike maintaining a separate roadmap for each orbit, each $G_i$ evolves across the discovered modes and orbits as a mode- and orbit-labeled single-robot planning graph shared by the global search. Together, these roadmaps conceptually induce the tensor-product roadmap
\begin{equation}
G_{\otimes}=G_1\otimes\cdots\otimes G_R,
\label{eq:tensor-roadmap}
\end{equation}
which is never explicitly constructed. The planner instead maintains a \emph{tensor-product search tree} $\mathbf{T}$ using the implicit product-neighbor expansion and rewiring mechanism of dRRT$^*$~\cite{shome2020drrt}; we refer to that work for details of the standard product-search construction.

The constituent roadmaps are constructed lazily. Candidate single-robot motions are not exhaustively validated when they are generated. Instead, their motion-validity checks are deferred until the corresponding edges are considered during tensor-product tree expansion, while robot--robot collisions are checked only for explored composite motions. Thus, geometric computation is concentrated on portions of the roadmaps actually examined by the multi-robot search.

An orbit may be reached through multiple transition configurations and may contain multiple candidate exits. All sampled transitions remain available to the global tensor-product search with their cost-to-come values, allowing different entries and exits to compete during search and rewiring rather than committing to the first discovered transition.

The multi-modal structure further guides where planning effort is allocated. Before a solution is available, \emph{optimistic modal guidance} biases transition sampling toward modes estimated to be closer to a goal mode. After a solution is found, \emph{orbit-path guidance} preferentially allocates sampling and tensor-tree expansion to orbits along the incumbent goal-reaching path to improve its cost. Both guidance mechanisms are combined with persistent exploration, so alternative orbit sequences continue to receive planning effort.

\begin{algorithm}[t]
\caption{Modal- and Orbit-Guided AO MR-TAMP}
\label{alg:ao-mrtamp}

\KwIn{Initial configuration $\mathbf{q}_s$, goal set $\mathcal{Q}_g$,
interior samples $N_m$, transition samples $N_t$}
\KwOut{Best solution $\Pi_{\mathrm{best}}$}

$\mathcal{O}_s \gets \mathrm{Orbit}(\mathbf{q}_s)$\;
$\mathcal{G}_O \gets \{\mathcal{O}_s\}$\;
$\mathcal{G}_M \gets \{\mathrm{ModalState}(\mathcal{O}_s)\}$\;
Initialize constituent roadmaps $\{G_i\}_{i=1}^{R}$\;
Initialize tensor-product search tree $\mathbf{T}$ at
$(\mathcal{O}_s,\mathbf{q}_s)$\;
$L_O \gets \emptyset,\;
 \Pi_{\mathrm{best}} \gets \emptyset$\;

\While{planning budget remains}{

    \eIf{$\Pi_{\mathrm{best}}\neq\emptyset$}{
        $(\mathcal{O}_{\mathrm{samp}},\Gamma_{\mathrm{samp}})
        \gets
        \mathrm{OrbitPathGuidance}(L_O,\mathcal{G}_O)$\;
    }{
        $S_{\mathrm{samp}}
        \gets \mathrm{ModalFrontier}(\mathcal{G}_M)$\;

        $\mathcal{O}_{\mathrm{samp}}
        \gets
        \mathrm{SelectOrbit}(S_{\mathrm{samp}},\mathcal{G}_O)$\;

        $\Gamma_{\mathrm{samp}}
        \gets
        \mathrm{OptimisticModalGuidance}
        (S_{\mathrm{samp}},\mathcal{G}_M,\mathcal{Q}_g)$\;
    }

    $Q_m \gets
    \mathrm{SampleOrbitInterior}
    (\mathcal{O}_{\mathrm{samp}},N_m)$\;

    $Q_t \gets
    \mathrm{SampleTransitions}
    (\mathcal{O}_{\mathrm{samp}},
     \Gamma_{\mathrm{samp}},N_t)$\;

    \ForEach{feasible transition $\mathbf{t}\in Q_t$}{
        $\mathcal{O}_{\mathrm{new}}
        \gets \mathrm{DiscoverAdjacentOrbit}(\mathbf{t})$\;

        Update $\mathcal{G}_O$ and $\mathcal{G}_M$ with
        $\mathcal{O}_{\mathrm{new}}$ and $\mathbf{t}$\;
    }

    $\mathrm{ExpandRoadmaps}
    (\{G_i\}_{i=1}^{R},
     \mathcal{O}_{\mathrm{samp}},Q_m,Q_t)$\;

    \While{extension budget remains}{

        $\mathbf{q}_{\mathrm{near}}
        \gets
        \mathrm{SelectTensorNode}
        (\mathbf{T},\mathcal{O}_{\mathrm{samp}},L_O)$\;

        $\mathbf{q}_{\mathrm{new}}\gets\mathrm{GetNeighbor}
        (\mathbf{q}_{\mathrm{near}},
         \mathcal{O}_{\mathrm{samp}},
         \{G_i\}_{i=1}^{R})$\;

        \If{$\mathrm{LazyValidate}
        (\mathbf{q}_{\mathrm{near}},\mathbf{q}_{\mathrm{new}})$}{
            $\mathrm{InsertAndRewire}
            (\mathbf{T},
             \mathbf{q}_{\mathrm{near}},
             \mathbf{q}_{\mathrm{new}})$\;
        }
    }

    $(\Pi,J)
    \gets \mathrm{BestSolution}
    (\mathbf{T},\mathcal{Q}_g)$\;

    \If{$\Pi\neq\emptyset$ \textbf{and}
        $(\Pi_{\mathrm{best}}=\emptyset
        \ \textbf{or}\ 
        J<J(\Pi_{\mathrm{best}}))$}{

        $\Pi_{\mathrm{best}}\gets\Pi$\;

        $L_O\gets
        \mathrm{ExtractOrbitPath}(\Pi_{\mathrm{best}})$\;
    }
}

\Return{$\Pi_{\mathrm{best}}$}\;

\end{algorithm}

Algorithm~\ref{alg:ao-mrtamp} summarizes the overall procedure. The guidance mechanisms affect the finite-time allocation of computation but do not exclude alternative orbit sequences. Section~\ref{sec:theory} shows that stratified transition coverage, persistent exploration, and asymptotically improving tensor-product motion search establish asymptotic optimality of Algorithm~\ref{alg:ao-mrtamp}.

\section{Asymptotic Optimality Analysis}
\label{sec:theory}

The general multi-modal AO analysis~\cite{shome2020pushing} establishes convergence when robust transition neighborhoods are sufficiently sampled and the continuous motions between them are asymptotically optimized. Its application to MR-TAMP is nontrivial, however, because the transition boundary is a stratified collection of manifolds of potentially different dimensions. We therefore first specialize the robust-transition argument to this geometry and then show that the sampling and continuous-search mechanisms of Algorithm~\ref{alg:ao-mrtamp} satisfy the resulting conditions.

We assume the standard robustness conditions~\cite{shome2020pushing}: informally, an optimal solution can be approximated arbitrarily closely by feasible solutions passing through positive-measure neighborhoods of its transitions. We additionally assume that the objective is consistent with the decomposition in Equation~\eqref{eq:additive-cost}.

\begin{assumption}[Transition-stratum structure]
\label{ass:strata}
Every transition boundary relevant to a robustly optimal solution can be decomposed locally into finitely many smooth manifold pieces, referred to as \emph{transition strata}. For every required transition, sufficiently small neighborhoods within its stratum have positive intrinsic measure.
\end{assumption}

\begin{assumption}[Transition-stratum coverage]
\label{ass:transition}
For every positive-intrinsic-measure transition neighborhood $B\subset \mathcal{T}^{A}_{\mathcal{O},\mathcal{O}'}$ required by a robustly optimal solution, there exists $\eta_B>0$ such that, whenever $\mathcal{O}$ is selected for transition sampling, the probability of sampling a transition in $B$ is at least $\eta_B$.
\end{assumption}

\begin{assumption}[Persistent exploration]
\label{ass:persistent}
Every discovered orbit belonging to a robustly optimal orbit sequence receives infinitely many transition-sampling, constituent-roadmap, and tensor-product-search expansion attempts almost surely.
\end{assumption}

\subsection{Robust MR-TAMP Transition Sequences}

The robust-transition argument~\cite{shome2020pushing} must be interpreted using the intrinsic geometry of each transition stratum.

\begin{lemma}[Robust transition-sequence approximation]
\label{lem:robust-seq}
For every $\epsilon>0$, suppose an optimal MR-TAMP solution satisfies the robustness conditions~\cite{shome2020pushing} and Assumption~\ref{ass:strata}. Then there exists a finite sequence of transition neighborhoods $B_1,\ldots,B_K,$ where $B_k\subset \mathcal{T}^{A_k}_{\mathcal{O}_{k-1},\mathcal{O}_k},$ each having positive intrinsic measure in its corresponding stratum, such that transitions sampled from these neighborhoods together with sufficiently accurate orbital motions admit a feasible MR-TAMP solution of cost at most $J^*+\epsilon$.
\end{lemma}

\begin{proof}
The multi-modal analysis~\cite{shome2020pushing} provides robust transition neighborhoods for a finite transition sequence. In MR-TAMP, each required transition lies on a particular stratum of the stratified transition set in Section~\ref{sec:geometry}. Under Assumption~\ref{ass:strata}, each robust transition neighborhood can be restricted to a sufficiently small neighborhood within the corresponding stratum while retaining positive intrinsic measure. Choosing these neighborhoods and the intervening orbital approximations sufficiently small yields a feasible solution whose cost approaches $J^*$ arbitrarily closely.
\end{proof}

\subsection{Transition Coverage and Orbital Convergence}

\begin{lemma}[Almost-sure transition coverage]
\label{lem:transition-coverage}
Under Assumptions~\ref{ass:transition} and~\ref{ass:persistent}, every
transition neighborhood $B_k$ in
Lemma~\ref{lem:robust-seq} is eventually sampled almost surely.
Consequently, every orbit along the corresponding finite robustly optimal
orbit sequence is eventually discovered almost surely.
\end{lemma}

\begin{proof}
Consider a required transition neighborhood $B_k \subset \mathcal{T}^{A_k}_{\mathcal{O}_{k-1},\mathcal{O}_k}.$ Suppose that its predecessor orbit $\mathcal{O}_{k-1}$ has been discovered. By Assumption~\ref{ass:persistent}, $\mathcal{O}_{k-1}$ receives infinitely many transition-sampling attempts almost surely.

By Assumption~\ref{ass:transition}, each such attempt has probability at least $\eta_{B_k}>0$ of sampling inside $B_k$. Conditioned on any previous sampling history, the probability of missing $B_k$ in the next $n$ such attempts is therefore at most $(1-\eta_{B_k})^n$, which converges to zero as $n\rightarrow\infty$. Hence $B_k$ is eventually sampled almost surely.

The initial orbit $\mathcal{O}_0$ is known by construction. Applying the argument to $B_1$ implies eventual discovery of $\mathcal{O}_1$ almost surely. Assumption~\ref{ass:persistent} applies to $\mathcal{O}_1$, and the same argument establishes eventual discovery of $\mathcal{O}_2$. Proceeding inductively along the finite orbit sequence establishes the claim.
\end{proof}

The constituent roadmaps of Algorithm~\ref{alg:ao-mrtamp} evolve across multiple modes and orbits rather than being constructed independently for each orbit. For analysis, we consider their restriction to the mode- and orbit-labeled configurations and motions associated with a relevant orbit
$\mathcal{O}$.

\begin{lemma}[Orbital motion convergence]
\label{lem:orbital-ao}
Suppose that, within every relevant orbit, the restricted constituent roadmaps satisfy the asymptotic coverage and connection conditions of their underlying AO single-robot roadmap planners, and that the implicit tensor-product tree satisfies the exploration and rewiring conditions required by dRRT$^*$~\cite{shome2020drrt}. Under Assumption~\ref{ass:persistent}, the best feasible motion represented between transition neighborhoods within every orbit on a robustly optimal sequence converges almost surely to its optimal robust orbital cost.
\end{lemma}

\begin{proof}
By Assumption~\ref{ass:persistent}, every relevant discovered orbit receives unbounded constituent-roadmap sampling and tensor-product-search effort. Restricting the shared constituent roadmaps to such an orbit produces growing roadmap subgraphs satisfying the assumed single-robot AO conditions. Their tensor product is searched implicitly using the exploration and rewiring mechanism of dRRT$^*$, yielding asymptotically optimal multi-robot motion within the orbit.

Lazy validation changes only when candidate motions are checked. Every candidate edge selected for use in the tensor-product search is validated before being accepted into a feasible solution, while invalid edges are rejected. Thus, lazy validation changes the order of collision checking without removing feasible motions from the asymptotic search space. The best feasible orbital cost therefore converges almost surely to the optimal robust cost.
\end{proof}

Optimistic modal guidance and orbit-path guidance affect only the finite-time allocation of planning effort. Because they are mixed with the persistent exploration policy required by Assumption~\ref{ass:persistent}, they may prioritize promising modes and incumbent-solution orbits without permanently excluding alternative orbit sequences.

\subsection{Global Asymptotic Optimality}

\begin{theorem}[Asymptotic optimality of Algorithm~\ref{alg:ao-mrtamp}]
\label{thm:global-ao}
Let $\Pi_n$ denote the best complete solution returned by Algorithm~\ref{alg:ao-mrtamp} after $n$ planning iterations. Under the robustness conditions~\cite{shome2020pushing}, Assumptions~\ref{ass:strata}--\ref{ass:persistent}, and the conditions of Lemma~\ref{lem:orbital-ao},
\begin{equation}
\Pr\left(
\lim_{n\rightarrow\infty}J(\Pi_n)=J^*
\right)=1.
\label{eq:global-ao}
\end{equation}
\end{theorem}

\begin{proof}
Fix any $\epsilon>0$. By Lemma~\ref{lem:robust-seq}, there exists a finite sequence of positive-intrinsic-measure transition neighborhoods and sufficiently accurate orbital motions admitting a feasible MR-TAMP solution with cost arbitrarily close to $J^*$. Lemma~\ref{lem:transition-coverage} ensures that all transition neighborhoods in this sequence, and hence all corresponding orbits, are eventually discovered almost surely. Lemma~\ref{lem:orbital-ao} then ensures asymptotic convergence of the
intervening orbital motion costs.

Thus, the transition-coverage and continuous-optimization conditions required by the multi-modal AO result~\cite{shome2020pushing} are satisfied by Algorithm~\ref{alg:ao-mrtamp}. Consequently, for every $\epsilon>0$ the planner eventually contains a feasible solution of cost at most $J^*+\epsilon$ almost surely. Since the incumbent cost is nonincreasing and no feasible solution can have cost below $J^*$, its cost converges almost surely to $J^*$.
\end{proof}

\section{Experiments}
\label{sec:experiments}
We evaluate the proposed framework on MR-TAMP problems designed to assess its practical planning performance. While our main contribution is the theoretical analysis of asymptotically optimal MR-TAMP, we also develop a practical algorithm that realizes the resulting planning principles. The experiments evaluate (i) the ability to find feasible solutions across multi-robot manipulation tasks, (ii) the effect of orbit- and modal-level guidance on planning performance, (iii) the anytime improvement of solution quality, and (iv) scalability with increasing numbers of robots and objects.

All planners are implemented in Python and evaluated in PyBullet on a workstation equipped with an Intel Core Ultra 9 285K CPU (24 cores) and 64~GB of RAM. All methods are evaluated using the same planning budget and environment settings. Each experiment is repeated over 5 random seeds, with the same set of seeds used across methods. We report success rate, time to the first feasible solution, first-solution cost, final best-solution cost, and solution-cost convergence over planning time.







\subsection{Baselines and Ablations}


We compare against two baselines. \textbf{Composed PRM$^\ast$}~\cite{hartmann2025sampling} constructs a roadmap directly in the composite multi-robot configuration space and is used to evaluate the benefit of constituent roadmaps and implicit tensor-product search. \textbf{Proposed (unguided)} is an ablation of the proposed method that removes orbit-frontier and modal-state guidance while retaining the same underlying planning framework.

\subsection{Benchmark Tasks}

We evaluate the planners on four multi-robot manipulation tasks, as illustrated in Fig.~\ref{fig:environments}. The simulation environments is using Franka Emika Panda robots with 7 DoF. The \textit{Box Pick-and-Place} task requires two robots to retrieve 2 objects from a box, where the limited workspace permits only one robot to pick at a time, and place them within the black target region on the table. The \textit{Handover} task requires coordinated object transfer between two robots and evaluates transitions involving robots interactions. The \textit{Stacking} task requires the robots to construct a target stack of 4 cubes. This task requires task order reasoning and introduces precedence constraints between object placements. For these three tasks, we run each algorithm for a maximum of 500 iterations with a time limit of 10 minutes. Each run terminates when either limit is reached.

The \textit{Simple Pick-and-Place} task requires multiple robots to pick objects from specified initial locations and place them at corresponding goal locations. Unlike the preceding tasks, this task is parameterized by the numbers of robots and objects and is used primarily for the scalability study. An example with three robots is shown in Fig.~\ref{fig:3robotpickplace}.

\begin{figure*}[t]
    \centering
    \subfloat[Box Pick-and-Place with sequential picking constraint.]{
        \includegraphics[width=0.23\textwidth]{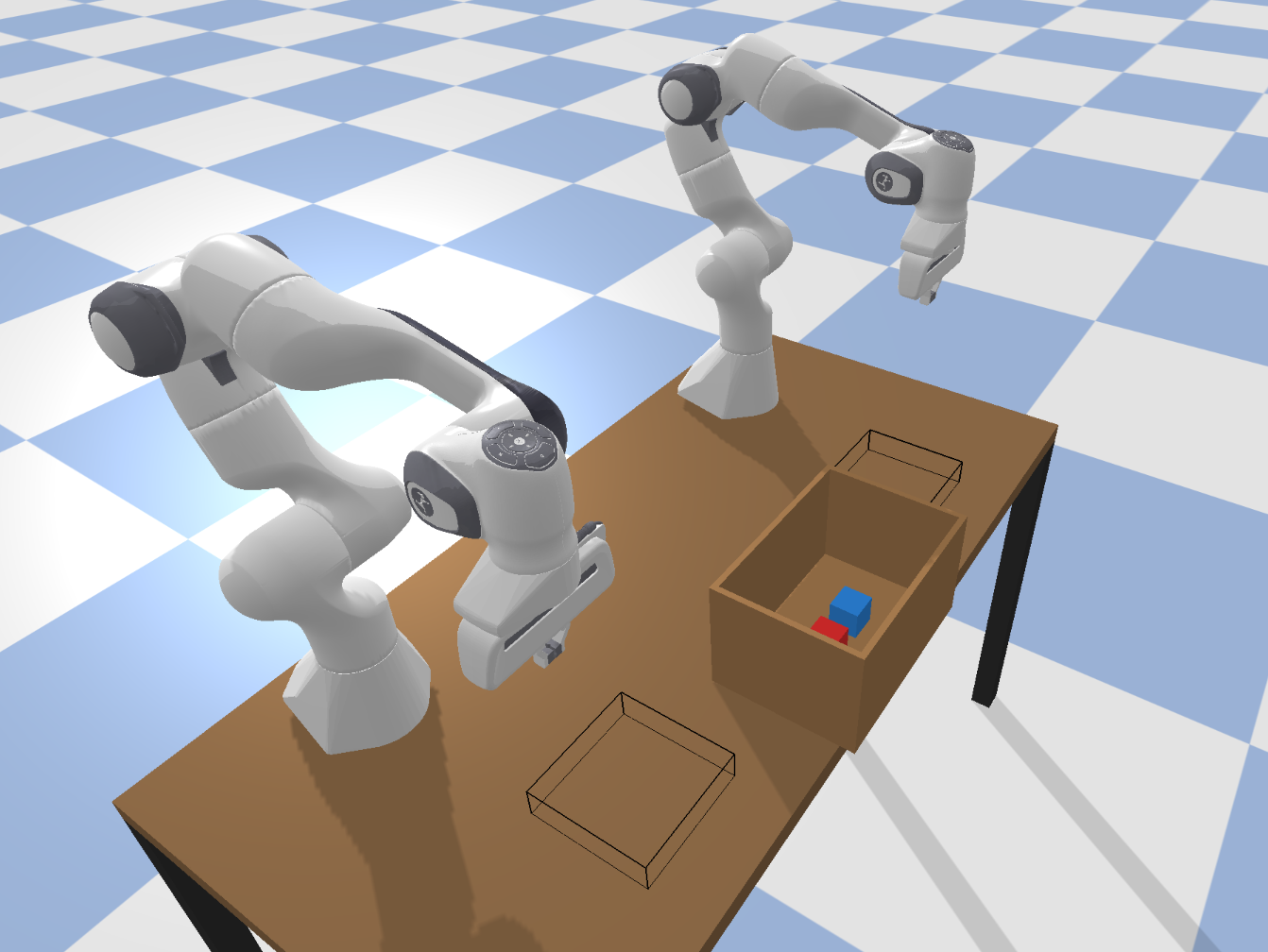}
        \label{fig:pick_place}
    }
    \hfill
    \subfloat[Handover with object goals indicated by corresponding colored dots.]{
        \includegraphics[width=0.23\textwidth]{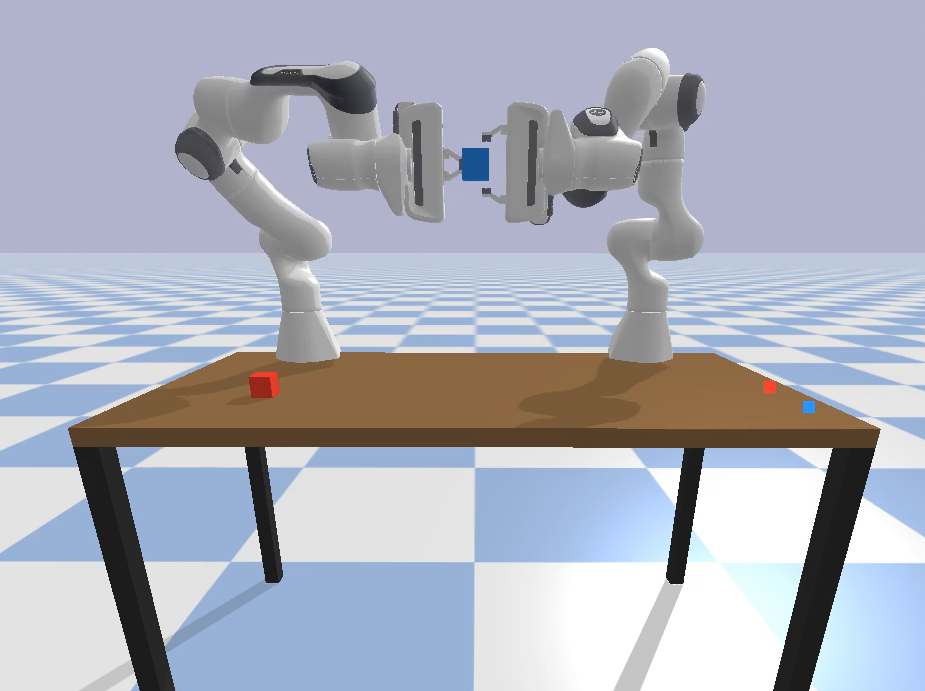}
        \label{fig:handover}
    }
    \hfill
    \subfloat[Stacking with the goal location indicated by the black dot.]{
        \includegraphics[width=0.23\textwidth]{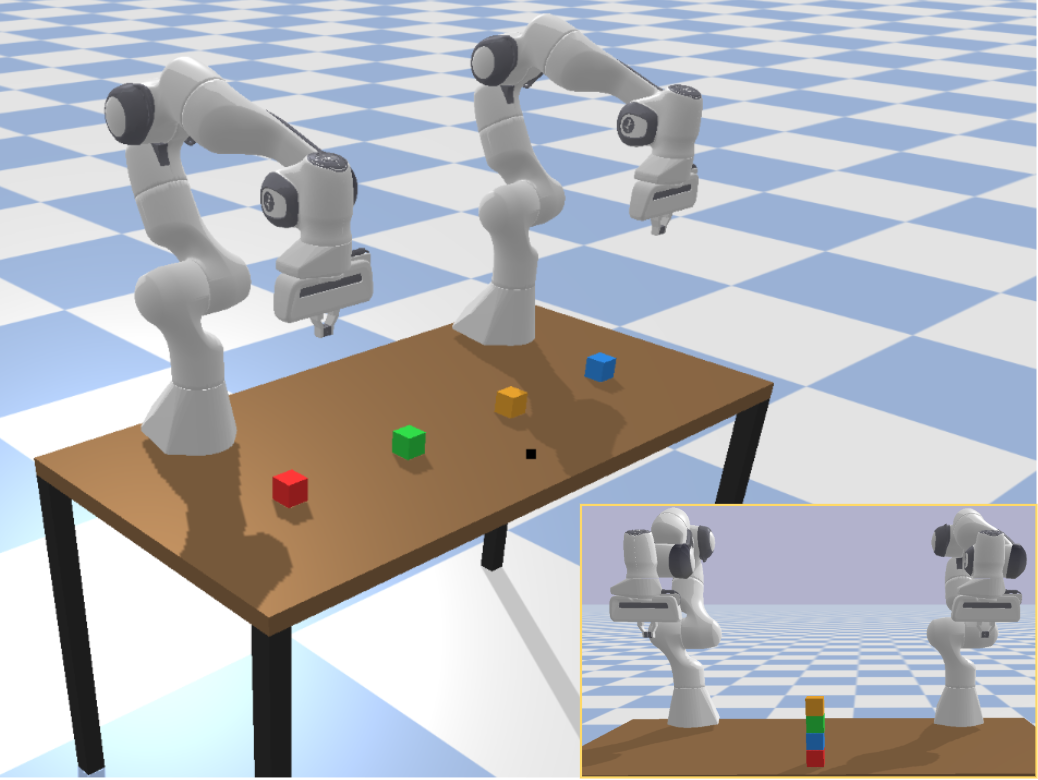}
        \label{fig:stacking}
    }
    \hfill
    \subfloat[Three robots pick and place cubes at the corresponding color-coded goals.]{
        \includegraphics[width=0.23\textwidth]{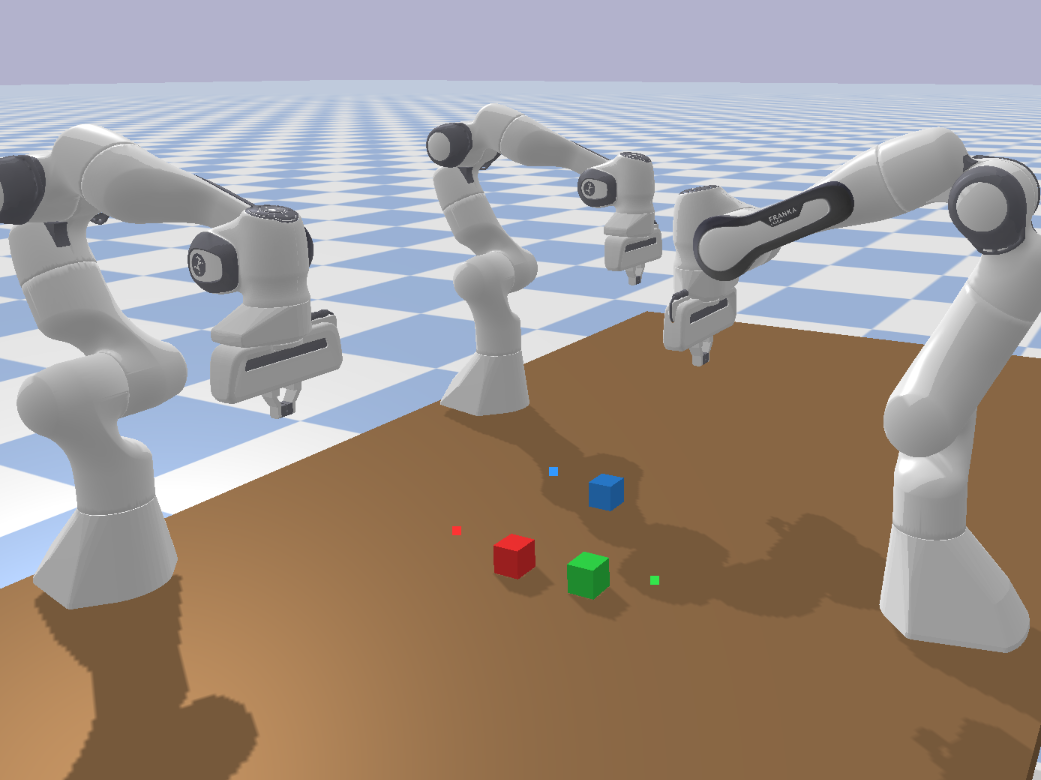}
        \label{fig:3robotpickplace}
    }
    \caption{\textbf{Multi-robot manipulation tasks} used in the experiments.}
    \label{fig:environments}
\end{figure*}

\subsection{Results}



\begin{figure}[t]
    \centering
    \subfloat[Final solution across random seeds.]{
        \includegraphics[width=0.9\linewidth]
        {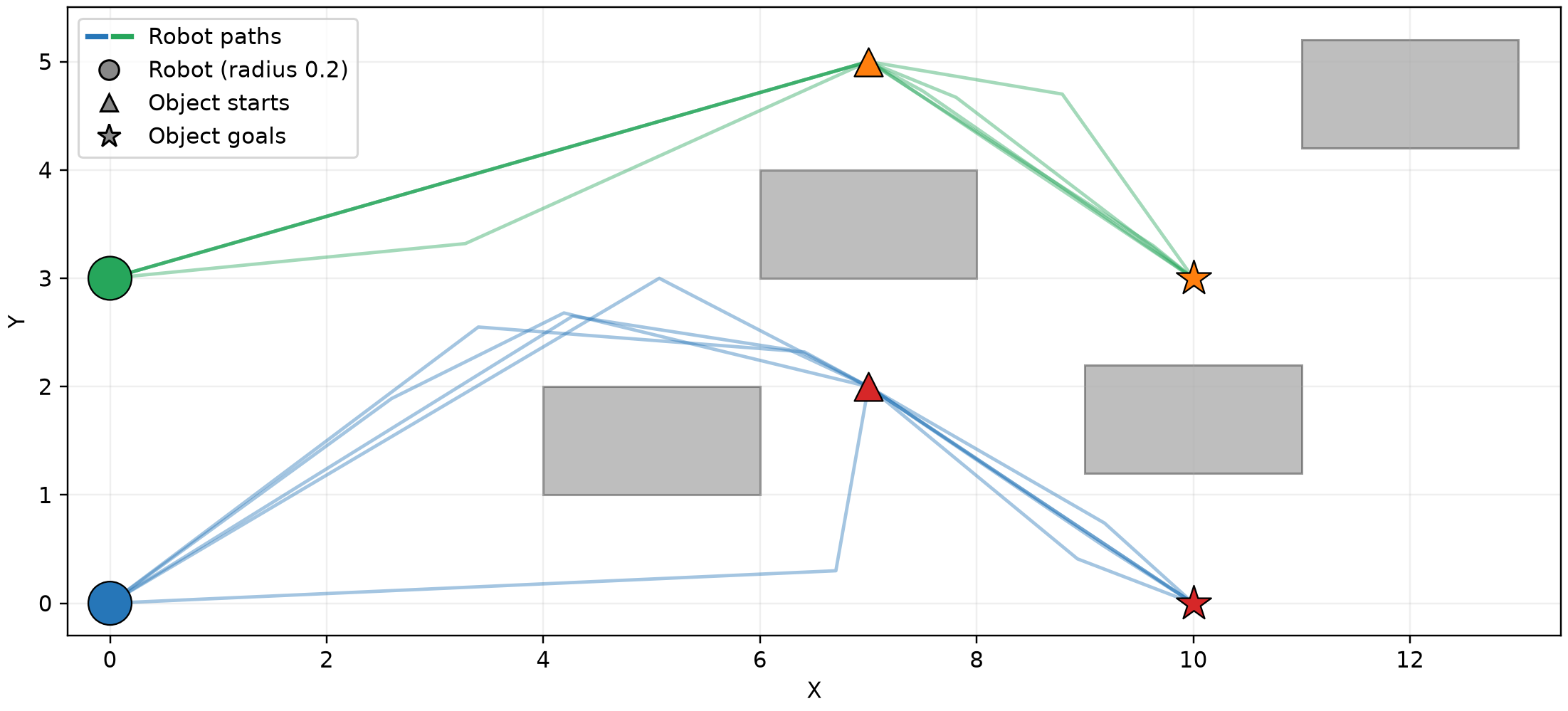}
        \label{fig:2d_solution}
    }
    \hfill
    \subfloat[Mean convergence 95\% CI.]{
        \includegraphics[width=0.9\linewidth]
        {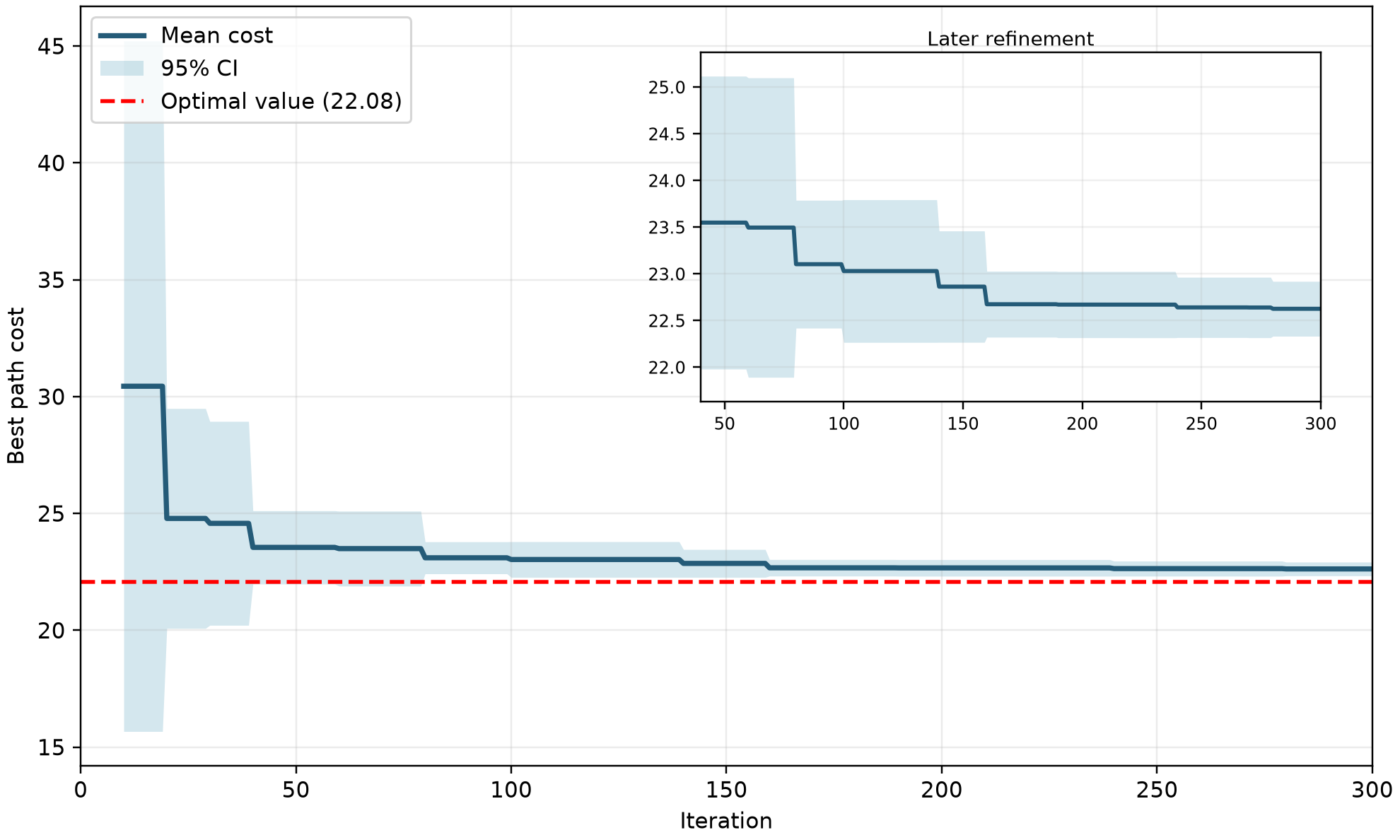}
        \label{fig:convergence}
    }
    \caption{\textbf{Results in the 2D multi-modal motion planning environment.}}
    \label{fig:2denv}
\end{figure}


Fig.~\ref{fig:2denv} illustrates the proposed planner in the 2D multi-modal motion planning environment with robot radius $r=0.2$. The planner finds collision-free paths through the intermediate subgoals, while the solution cost decreases after the first feasible solution and is subsequently refined (Fig.~\ref{fig:convergence}). The mean cost approaches the known optimum of approximately $22.08$, demonstrating the anytime behavior of the planner.

Results on the Franka environments are reported in Tables~\ref{tab:main_results} and~\ref{tab:scalability}. Overall, the proposed method demonstrates improved reliability, particularly for problems with longer manipulation horizons. For the 2-robot/6-object case,
it achieves a 100\% success rate, compared with 40\% for Composed PRM$^\ast$ and 0\% for Proposed (unguided). Similarly, for the 3-robot/6-object case, the success rates are 100\%, 80\%, and 0\%, respectively. These results indicate that high-level guidance becomes increasingly important as the number of required manipulation transitions grows.

Composed PRM$^\ast$ remains competitive on smaller instances, where explicit sampling of the composite configuration space can quickly establish connectivity. In contrast, the proposed formulation uses lower-dimensional constituent roadmaps and searches their tensor product implicitly, avoiding explicit construction of the composite roadmap. However, the performance of Proposed (unguided) on longer-horizon problems shows that the implicit representation alone does not guarantee efficient exploration. The proposed orbit-frontier and modal-state guidance directs the search toward promising orbits and transitions, substantially improving reliability on these problems.

The proposed method also generally finds feasible solutions faster on the more challenging problems and achieves lower final costs in most cases. The incumbent orbit path further guides subsequent exploration toward regions relevant to improving the current solution while retaining exploration of alternative orbits.

Finally, the transition sampler discovers both individual and coupled robot transitions, including the two-robot transition required for handover. Scalability nevertheless becomes more challenging as the number of robots increases, since more active-robot subsets must be considered and many coupled transitions are geometrically infeasible, increasing the sampling and collision-checking effort.





\begin{table*}[t]
\centering
\caption{Planning performance on the Box Pick-and-Place, Handover, and Stacking tasks. SR denotes the success rate, $T_{\mathrm{first}}$ the time to the first feasible solution, $J_{\mathrm{first}}$ the cost of the first solution, and $J_{\mathrm{final}}$ the final best solution
cost. Time and cost are reported as mean $\pm$ standard deviation over multiple random seeds. Higher SR and lower time and cost indicate better performance.}
\label{tab:main_results}
\resizebox{\textwidth}{!}{
\begin{tabular}{lcccccccccccc}
\toprule
& \multicolumn{4}{c}{Box Pick-and-Place}
& \multicolumn{4}{c}{Handover}
& \multicolumn{4}{c}{Stacking} \\
\cmidrule(lr){2-5}
\cmidrule(lr){6-9}
\cmidrule(lr){10-13}

Method
& SR $\uparrow$
& $T_{\mathrm{first}}$ $\downarrow$
& $J_{\mathrm{first}}$ $\downarrow$
& $J_{\mathrm{final}}$ $\downarrow$
& SR $\uparrow$
& $T_{\mathrm{first}}$ $\downarrow$
& $J_{\mathrm{first}}$ $\downarrow$
& $J_{\mathrm{final}}$ $\downarrow$
& SR $\uparrow$
& $T_{\mathrm{first}}$ $\downarrow$
& $J_{\mathrm{first}}$ $\downarrow$
& $J_{\mathrm{final}}$ $\downarrow$ \\
\midrule

Composed PRM$^\ast$
& 100\% & $11\pm 5$ & $46.3\pm 5.2$ & $31.7\pm 1.1$
& 100\% & $23\pm 9$ & $56.7\pm 9.4$ & $41.3\pm 1.1$
& 80\% & $121\pm 31$ & $63.2\pm 7.6$ & $57.6\pm 5.6$ \\

Proposed (unguided)
& 100\% & $13\pm 6$ & $46.9\pm 10.7$ & $33.8\pm 2.5$
& 100\% & $31\pm 15$ & $51\pm 1.3$ & $42.1\pm 2.9$
& 20\% & 242.5 & 71.1 & 65.1 \\


Proposed
& 100\% & $10\pm 3$ & $42.6\pm 3.4$ & $30.0\pm 2.3$
& 100\% & $12\pm 3$ & $49.5\pm 6.2$ & $35.4\pm 1.6$
& 100\% & $47\pm 37$ & $66.1\pm 2.7$ & $51.3\pm 2.3$ \\

\bottomrule
\end{tabular}
}
\end{table*}









\begin{table*}[t]
\centering
\caption{Scalability comparison on the Simple Pick-and-Place task with increasing numbers of robots and objects.}
\label{tab:scalability}
\resizebox{\textwidth}{!}{
\begin{tabular}{cc|ccc|ccc|ccc}
\toprule
\multirow{2}{*}{Robots}
& \multirow{2}{*}{Objects}
& \multicolumn{3}{c|}{Composed PRM$^\ast$}
& \multicolumn{3}{c|}{Proposed (unguided)}
& \multicolumn{3}{c}{Proposed} \\
\cmidrule(lr){3-5}
\cmidrule(lr){6-8}
\cmidrule(lr){9-11}

&
& SR $\uparrow$
& $T_{\mathrm{first}}$ $\downarrow$
& $J_{\mathrm{final}}$ $\downarrow$

& SR $\uparrow$
& $T_{\mathrm{first}}$ $\downarrow$
& $J_{\mathrm{final}}$ $\downarrow$


& SR $\uparrow$
& $T_{\mathrm{first}}$ $\downarrow$
& $J_{\mathrm{final}}$ $\downarrow$
\\
\midrule
2 & 2 & 100\% & $5\pm 3$ & $19.8\pm 0.9$ & 100\% & $6\pm 5$ & $22.9\pm 2.8$ & 100\% & $4\pm2$ & $19.1\pm 0.7$ \\
2 & 4 & 100\% & $26\pm 19$ & $46.2\pm 2.0$ & 80\% & $51\pm 39$ & $57.3\pm 3.8$ & 100\% & $30\pm33$ & $47.7\pm 4.4$ \\
2 & 6 & 40\% & $106\pm 14$ & $72.7\pm 0.4$ & 0\% & -- & -- & 100\% & $53\pm25$ & $78.2\pm 11.4$ \\
\midrule
3 & 3 & 100\% & $10\pm 5$ & $45.8\pm 5.6$ & 100\% & $50\pm 23$ & $43\pm 4.5$ & 100\% & $7\pm5$ & $41.7\pm 3.2$ \\
3 & 6 & 80\% & $157\pm 58$ & $113.1\pm 9.6$ & 0\% & -- & -- & 100\% & $111\pm 88$ & $97.2\pm 9.9$ \\
\midrule
4 & 4 & 100\% & $37\pm 23$ & $72.5\pm 7.6$ & 100\% & $103\pm 53$ & $95.6\pm 4.3$ & 100\% & $57\pm 33$ & $67.4\pm 10.3$ \\
4 & 6 & 80\% & $257\pm 159$ & $160.8\pm 1.5$ & 0\%- & -- & -- & 80\% & $176\pm 158$ & $140.4\pm 21.3$ \\

\bottomrule
\end{tabular}
}
\end{table*}

\subsection{Real Robot Experiments}

We further validate the proposed planner on a real-world setup consisting of two Franka Research 3 robotic manipulators operating in a shared workspace. We evaluate the planner on the \textit{Pick-and-Place}, \textit{Handover}, and \textit{Stacking} tasks. These tasks demonstrate different planning capabilities, including shared-workspace manipulation, coupled robot transitions, and precedence-constrained manipulation.

The physical environment is first reproduced in PyBullet, where the proposed planner performs motion planning and collision checking. The resulting collision-free multi-robot trajectories are then transferred to the physical robots for execution. The experiments demonstrate the feasibility of executing the planned multi-modal sequences and coordinated collision-free motions on the real robotic system.






\section{Conclusion}

We study asymptotically optimal planning for multi-robot task and motion planning.
The key geometric observation is that MR-TAMP mode boundaries naturally contain transition strata of different dimensions because different subsets of robots may participate in a modal transition.
We use this structure to define a conditional transition sampler and identify sufficient conditions under which relevant transition sequences are discovered while continuous motion inside each orbit converges to its optimum.
By combining these properties with existing AO-TAMP theory, we establish global asymptotic optimality for the resulting MR-TAMP framework.
We further describe a practical implementation based on dRRT$^*$ that searches multi-robot product roadmaps implicitly.
Together, these results provide a theoretical foundation for asymptotically optimal MR-TAMP while allowing scalable multi-robot motion-planning techniques to be used as interchangeable intra-orbit planners.

\bibliographystyle{IEEEtran}
\bibliography{references}

\end{document}